\documentclass{article}

\PassOptionsToPackage{numbers, compress}{natbib}

\usepackage[preprint]{neurips}

\usepackage{xcolor}         
\definecolor{sbblue}{HTML}{4878d0}
\definecolor{sbred}{HTML}{d65f5f}
\definecolor{sbgreen}{HTML}{6acc64}
\definecolor{sbbluedeep}{HTML}{4c72b0}
\definecolor{sbreddeep}{HTML}{c44e52}
\definecolor{sbgreendeep}{HTML}{55a868}

\usepackage[utf8]{inputenc} 
\usepackage[T1]{fontenc}    
\usepackage{hyperref}       
\hypersetup{
  colorlinks,
  citecolor=sbgreen,
  linkcolor=sbred,
  urlcolor=sbblue}
\usepackage{url}            
\usepackage{booktabs}       
\usepackage{amsfonts}       
\usepackage{nicefrac}       
\usepackage{microtype}      
\usepackage{enumitem}       
\usepackage{graphicx}       
\usepackage{siunitx}        
\newcommand{\sd}[1]{\textsubscript{\color{gray}{\textpm \SI[round-precision=1, round-mode=places]{#1}{}}}}
\usepackage{tikz}           
\usetikzlibrary{shapes, positioning, chains, calc, arrows.meta}
\usepackage{wrapfig}        
\usepackage{adjustbox}      

\usepackage{amsthm,xpatch}  
\newtheorem{theorem}{Theorem}
\theoremstyle{remark}

\theoremstyle{definition}
\newtheorem{definition}{Definition}
\makeatletter
\newcounter{proofpart}[proof] 
\newcommand{\proofpart}{%
  \refstepcounter{proofpart}
  \noindent\emph{Step \theproofpart.}
}
\makeatother                

\usepackage{amsmath}        
\DeclareMathOperator{\supp}{supp}
\DeclareMathOperator{\rank}{rank}

\usepackage{xspace}         
\makeatletter
\DeclareRobustCommand\onedot{\futurelet\@let@token\@onedot}
\def\@onedot{\ifx\@let@token.\else.\null\fi\xspace}

\def\eg{\emph{e.g}\onedot} 
\def\ie{\emph{i.e}\onedot}

\def\wrt{w.r.t\onedot} 

\makeatother

\title{Compositional Generalization from First Principles}

\author{%
    Thaddäus Wiedemer$^{1,2,3}$\footnotemark[1] \quad Prasanna Mayilvahanan$^{1,2,3}$\footnotemark[1]\\
    \\
    \textbf{Matthias Bethge}$^{1,2}$\footnotemark[2] \quad \textbf{Wieland Brendel}$^{2,3}$\footnotemark[2]\\
    \\
    $^1$University of Tübingen \quad $^2$Tübingen AI Center\\
    $^3$Max-Planck-Institute for Intelligent Systems, Tübingen\\
    \\
    {\tt\small \{thaddaeus.wiedemer, prasanna.mayilvahanan\}@uni-tuebingen.de}
}

\begin{document}
\renewcommand*{\thefootnote}{\fnsymbol{footnote}}
\footnotetext[1]{Equal contribution \quad \footnotemark[2]Equal supervision\vspace{6pt}}
\footnotetext[0]{Code available at \url{github.com/brendel-group/compositional-ood-generalization}}

\maketitle

\begin{abstract}
Leveraging the compositional nature of our world to expedite learning and facilitate generalization is a hallmark of human perception. In machine learning, on the other hand, achieving compositional generalization has proven to be an elusive goal, even for models with explicit compositional priors. To get a better handle on compositional generalization, we here approach it from the bottom up: Inspired by identifiable representation learning, we investigate compositionality as a property of the data-generating process rather than the data itself. This reformulation enables us to derive mild conditions on only the support of the training distribution and the model architecture, which are sufficient for compositional generalization. We further demonstrate how our theoretical framework applies to real-world scenarios and validate our findings empirically. Our results set the stage for a principled theoretical study of compositional generalization.
\end{abstract}

\section{Introduction}\label{sec:intro}


\textit{Systematic compositionality}~\cite{fodorConnectionismCognitiveArchitecture1988} is the remarkable ability to utilize a finite set of known components to understand and generate a vast array of novel combinations. This ability, referred to by Chomsky~\cite{chomsky2014aspects} as the ``\textit{infinite use of finite means}'', is a distinguishing feature of human cognition, enabling us to adapt to diverse situations and learn from varied experiences.

It's been a long-standing idea to leverage the compositional nature of the world for learning. In object-centric learning, models learn to isolate representations of individual objects as building blocks for complex scenes. In disentanglement, models aim to infer factors of variation that capture compositional and interpretable aspects of their inputs, for example hair color, skin color, and gender for facial data. So far, however, there is little evidence that these methods deliver substantially increased learning efficacy or generalization capabilities (\citet{schottVisualRepresentationLearning2022}, \citet{monteroRoleDisentanglementgeneralisation2022}). Across domains and modalities, machine learning models still largely fail to capture and utilize the compositional nature of the training data (\citet{lake2018generalization, loula2018rearranging, keysers2019measuring}).


To exemplify this failure, consider a model trained on a data set with images of two sprites with varying position, size, shape, and color overlaid on a black canvas. Given the latent factors, a simple multi-layer neural network can easily learn to reconstruct images containing \textit{compositions} of these sprites that were covered by the training set (Figure~\ref{fig:motivation}, top rows). However, reconstruction fails for novel compositions---even if the individual \textit{components} have been observed before (Figure~\ref{fig:motivation}, bottom row). Failure to generalize to unseen data in even this simplistic regression setting demonstrates that \textit{compositional generalization} does not automatically emerge simply because the data is of a compositional nature.

We therefore take a step back to formally study compositionality and understand what conditions need to be fulfilled for compositional generalization to occur. To this end, we take inspiration from identifiable representation learning and define a broad class of data generating processes that are compositional and for which we can provably show that inference models can generalize to novel compositions that have not been part of the training set. More precisely, our contributions are as follows:
\begin{itemize}
    \item We specify \textit{compositional data-generating processes} both in terms of their function class and latent distributions (Sections~\ref{sec:compositionality} and \ref{sec:generalization}) such that they cover a wide range of assumptions made by existing compositional methods. 
    
    \item We prove a set of sufficient conditions under which models trained on the data are able to generalize compositionally (Section~\ref{sec:conditions}).
    
    \item We validate our theory in a range of synthetic experiments and perform several ablation studies that relate our findings to empirical methods (Section~\ref{sec:experiments}).
\end{itemize}

\begin{figure}[t]
  \centering
  \includegraphics[width=0.95\linewidth]{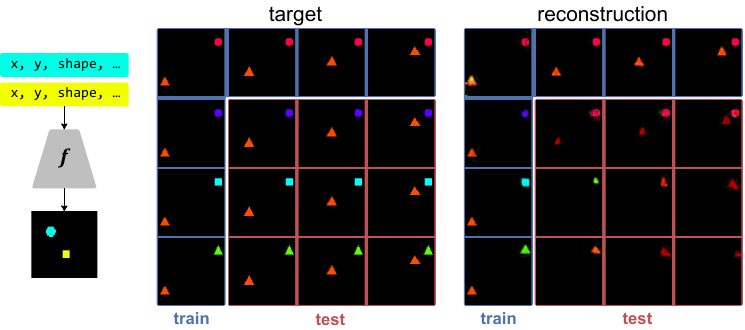}
  \caption{\textbf{Left}: We train a model $f$ to reconstruct images containing two sprites given their latent representation \texttt{(x, y, shape, size, color)}. \textbf{Center}: In the \textcolor{sbbluedeep}{\textbf{training set}} (top row and left column), one sprite is fixed in its base configuration (orange triangle or red circle), while the other can vary freely (in this example, sprite~1 varies in position, sprite~2 in shape and color). As a result, each sample in the \textcolor{sbreddeep}{\textbf{test set}} (lower right block) can be expressed as a novel \textit{composition} of known \textit{components}. \textbf{Right}: While the model is able to perfectly fit the training data, it fails to \textit{generalize compositionally} to the test data.}\label{fig:motivation}
\end{figure}

\section{Related Work}\label{sec:related}
\paragraph{Representation learning}
\textit{Disentanglement} and \textit{identifiable representation learning} aim to learn succinct representations that both factorize the data space efficiently and are robust towards distributional changes \cite{locatello2019fairness, 6472238, NEURIPS2019_bc3c4a63}. However, the expectation that more compositional representations lead to better out-of-distribution (OOD) generalization has not been met, as demonstrated by \citet{schottVisualRepresentationLearning2022} and \citet{monteroLostLatentSpace2022}. Although our work does not directly address generalization issues in identifiable representation learning, our setup is directly inspired by it, and we examine data-generating processes similar to \cite{reizinger2022embrace, zimmermann2021contrastive, von2021self}.

\paragraph{Empirical Approaches}
Many empirical methods use compositional priors and claim improved compositional generalization. The problem has been studied especially closely in language~\cite{russin2019compositional, akyurek2020learning, marcheggiani2018exploiting}, but it remains far from being solved~\cite{lake2018generalization, loula2018rearranging, keysers2019measuring}. Object-centric learning is another domain in which compositionality plays a major role, and many approaches explicitly model the composition of scenes from object-``slots''~\cite{locatelloObjectCentricLearningSlot2020, hinton2018matrix, greff2019multi, burgessMONetUnsupervisedScene2019, kipf2021conditional}. The slot approach is also common in vector-symbolic architectures like \cite{vankovTrainingNeuralNetworks2020} and \cite{fradyLearningGeneralizationCompositional2023}. For most of these works, however, compositional generalization is not a focal point, and their actual generalization capability remains to be studied. There are also some architectures like transformers~\cite{vaswaniAttentionAllYou2017}, graph neural networks~\cite{cappartcombinatorial}, bilinear models~\cite{hong2021bi}, or complex-valued autoencoders~\cite{loweComplexValuedAutoencodersObject2022} that have been claimed to exhibit some degree of compositional generalization, but again, principled analysis of their generalization ability is lacking. Our framework can guide the systematic evaluation of these methods. While we use the visual domain as an example throughout this work, our contributions are not tied to any specific data domain or modality. 

\paragraph{Theoretical approaches to OOD generalization}
The OOD generalization problem for non-linear models where train and test distributions differ in their densities, but not their supports, has been studied extensively, most prominently by \citet{ben2014domain} and \citet{sugiyama2007covariate}. We refer the reader to \citet{shenOutOfDistributionGeneralizationSurvey2021} for a comprehensive overview. In contrast, compositional generalization requires generalizing to a distribution with different, possibly non-overlapping support. This problem is more challenging and remains unsolved. \citet{ben2010theory} were able to show that models can generalize between distributions with a very specific relation, but it is unclear what realistic distributions fit their constraints. \citet{netanyahuLearningExtrapolateTransductive2023} also study \textit{out-of-support} problems theoretically but touch on compositional generalization only as a workaround for general extrapolation. Recently, \citet{dongFirstStepsUnderstanding2022} took a first step towards a more applicable theory of compositional generalization to unseen domains, but their results still rely on specific distributions, and they do not consider functions with arbitrary (nonlinear) compositions or multi-variate outputs. In contrast, our framework is independent of the exact distributions used for training and testing, and our assumptions on the compositional nature of the data allow us to prove generalization in a much broader setting. 

\section{A framework for compositional generalization}\label{sec:theory}

We use the following notation throughout. $[N]$ denotes the set of natural numbers $\{1, 2, ..., N\}$. $\boldsymbol{\mathrm {Id}}$~denotes the (vector-valued) identity function. We denote two functions $f, g$ agreeing for all points in set $P$ as $f {\equiv}_P \;g$. Finally, we write the total derivative of a vector-valued function $\boldsymbol f$ by all its inputs $\boldsymbol z$ as $\frac{\partial \boldsymbol f}{\partial \boldsymbol z}$, corresponding to the Jacobian matrix with entries $\frac{\partial f_i}{\partial z_j}$.





\subsection{Compositionality}\label{sec:compositionality}
Colloquially, the term ``\emph{compositional data}'' implies that the data can be broken down into discrete, identifiable components that collectively form the whole. For instance, in natural images, these components might be objects, while in music, they might be individual instruments. As a running illustrative example, we will refer to a simple dataset similar to multi-dSprites~\cite{burgessMONetUnsupervisedScene2019}, as shown in Figure~\ref{fig:motivation}. Each sample in this dataset is a composition of two basic sprites, each with a random position, shape, size, and color, size.

Drawing inspiration from identifiable representation learning, we define compositionality mathematically as a property of the data-generating process. In our example, the samples are generated by a simple rendering engine that initially renders each sprite individually on separate canvases. These canvases are then overlaid to produce a single image featuring two sprites. More specifically, the rendering engine uses the (latent) properties of sprite one, $\boldsymbol z_1 = (z_{1,\text{x}}, z_{1,\text{y}}, z_{1,\text{shape}}, z_{1,\text{size}}, z_{1,\text{color}})$, to produce an image $\boldsymbol{\tilde x}_1$ of the first sprite. The same process is repeated with the properties of sprite two, $\boldsymbol z_2 = (z{2,\text{x}}, z_{2,\text{y}}, z_{2,\text{shape}}, z_{2,\text{size}}, z_{2,\text{color}})$, to create an image $\boldsymbol{\tilde x}_2$ of the second sprite. Lastly, the engine combines $\boldsymbol{\tilde x}_1$ and $\boldsymbol{\tilde x}_2$ to create the final overlaid rendering $\boldsymbol x$ of both sprites. Figure~\ref{fig:function_class} demonstrates this process.

In this scenario, the individual sprite renderers carry out the bulk of the work. In contrast, the composition of the two intermediate sprite images $\boldsymbol{\tilde x}_1, \boldsymbol{\tilde x}_2$ can be formulated as a simple pixel-wise operation (see Appendix~\ref{app:theory_details} for more details). The rendering processes for each sprite are independent: adjusting the properties of one sprite will not influence the intermediate image of the other, and vice versa.

We posit that this two-step generative procedure---the (intricate) generation of individual components and their (simple) composition into a single output---is a key characteristic of a broad class of compositional problems. If we know the composition function, then understanding the basic elements (for example, the individual sprites) is enough to grasp all possible combinations of sprites in the dataset.
We can thus represent any latent variable model ${\boldsymbol f} : \mathcal{Z} \to \mathcal{X}$, which maps a latent vector ${\boldsymbol z}\in\mathcal{Z}$ to a sample $\boldsymbol x$ in the observation space $\mathcal{X}$, as a two-step generative process.

\begin{definition}[Compositional representation]\label{def:function_class}
$\{\boldsymbol C, \boldsymbol \varphi_1, \dots, \boldsymbol \varphi_K, \mathcal{Z}_1, \dots, \mathcal{Z}_K, \tilde{\mathcal{X}}_1, \dots, \tilde{\mathcal{X}}_K\}$ is a \textit{compositional representation} of function $\boldsymbol f$ if 
\begin{equation}
    \forall\boldsymbol{z}\in\mathcal{Z}\quad \boldsymbol f(\boldsymbol z) = \boldsymbol C \big( \boldsymbol \varphi_1(\boldsymbol z_1), ..., \boldsymbol \varphi_K(\boldsymbol z_K) \big)\quad\text{and}\quad\mathcal{Z} = \mathcal{Z}_1\times\dots\times\mathcal{Z}_K,
\end{equation}
where $\boldsymbol{z}_i$ denotes the canonical projection of $\boldsymbol{z}$ onto $\mathcal{Z}_i$. We refer to $\boldsymbol \varphi_k: \mathcal Z_k \to \tilde{\mathcal X}_k$ as the \textit{component functions}, to $\tilde{\mathcal X}_1, \dots, \tilde{\mathcal X}_K$ as the (hidden) component spaces, and to $\boldsymbol C: \tilde{\mathcal X}_1 \times \dots \times \tilde{\mathcal X}_K \to \mathcal X$ as the \textit{composition function}.
\end{definition}

Note that in its most general form, we do not require the component functions to be identical or to map to the same component space. The compositional representation of a function $\boldsymbol{f}$ is also not unique. For instance, any $\boldsymbol{f}$ possesses a trivial compositional representation given by $\{\boldsymbol{f}, \boldsymbol{\mathrm {Id}}, \dots, \boldsymbol{\mathrm {Id}}\}$ (for the sake of clarity, we will omit the explicit mention of the latent factorization and component spaces henceforth). We will later establish conditions that must be met by at least one compositional representation of $\boldsymbol{f}$.

\begin{figure}[t]
  \centering
  \begin{minipage}[b]{.425\textwidth}
      \centering
      \begin{adjustbox}{max width=\textwidth}
      \begin{tikzpicture}
    \begin{scope}[minimum size=9mm, inner sep=0]
      \node[circle, fill=lightgray](z1) at (0, 3) {$\boldsymbol z_1$};
      \node[circle, fill=lightgray](z2) at (0, 1.8) {$\boldsymbol z_2$};
      \node[circle, fill=lightgray](zK) at (0, 0) {$\boldsymbol z_K$};

      \node[circle, fill=lightgray](zt1) at (3, 3) {$\boldsymbol{\tilde x}_1$};
      \node[circle, fill=lightgray](zt2) at (3, 1.8) {$\boldsymbol{\tilde x}_2$};
      \node[circle, fill=lightgray](ztK) at (3, 0) {$\boldsymbol{\tilde x}_K$};

      \coordinate(C) at (5, 1.5);
      \node[circle, fill=lightgray](x) at (6, 1.5) {$\boldsymbol{x}$};
    \end{scope}

    \node at (1.5, 1.15){$\vdots$};

    \begin{scope}[out=0, in=180, line width=1pt]
        \draw (z1) edge[-latex] node[above]{$\boldsymbol \varphi_1$} (zt1) (zt1) edge (C);
        \draw (z2) edge[-latex] node[above]{$\boldsymbol \varphi_2$} (zt2) (zt2) edge (C);
        \draw (zK) edge[-latex] node[above]{$\boldsymbol \varphi_K$} (ztK) (ztK) edge (C);
        \draw (C) edge[-latex] node[above, pos=0]{$\boldsymbol C$}(x);
    \end{scope}
\end{tikzpicture}
      \end{adjustbox}
      \vspace{8mm}
      \caption{\textit{Compositional representation} of a function (Definition~\ref{def:function_class}). \textit{Component functions} $\boldsymbol \varphi_k$ map each \textit{component latent} $\boldsymbol z_k$ to an intermediate representation $\boldsymbol{\tilde x}_k$. The \textit{composition function} $\boldsymbol C$ composes them into a final data point $\boldsymbol x$.}\label{fig:function_class}
  \end{minipage}%
  \hspace{.04\textwidth}
  \begin{minipage}[b]{.525\textwidth}
      \centering
      \includegraphics[width=\textwidth]{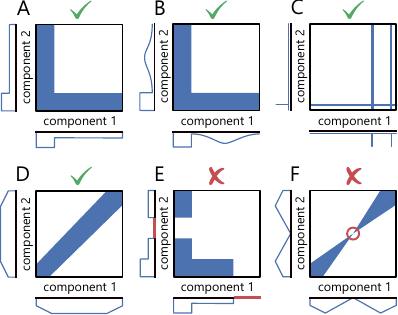}
      \caption{\textbf{A-D} Distribution $P$ (blue) has \textit{compositional support} \wrt to the entire latent space if it has full support over the marginals. \textbf{E} Gaps in the support require the model to interpolate/extrapolate rather than generalize compositionally. \textbf{F} The support of the joint needs to be in an open set.}\label{fig:comp_support}
  \end{minipage}
\end{figure}




Our definition of compositionality naturally aligns with various methods in the fields of identifiability, disentanglement, or object-centric learning. In the decoder of SlotAttention~\cite{locatelloObjectCentricLearningSlot2020}, for example, each component function is a spatial broadcast decoder followed by a CNN, and the composition function is implemented as alpha compositing.
\citet{fradyLearningGeneralizationCompositional2023} model the component functions as element-wise multiplication of high-dimensional latent codes, which are then composed through a straightforward sum. A similar approach is chosen by \citet{vankovTrainingNeuralNetworks2020}, except that interactions between components are modeled using matrix multiplication.

\subsection{Compositional Generalization}\label{sec:generalization}

The model in Figure~\ref{fig:motivation} was trained supervisedly, \ie it was trained to reconstruct samples $\boldsymbol{x}$ given the ground-truth latent factors $(\boldsymbol{z}_1, \boldsymbol{z}_2)$ for each sprite (see Section~\ref{sec:experiments} for more details). We denote this model as $\boldsymbol{\hat f}$, indicating that it is meant to replicate the ground-truth generating process $\boldsymbol{f}$ of the data. The model $\boldsymbol{\hat f}$ indeed learned to fit $\boldsymbol{f}$ almost perfectly on the training distribution $P$, but failed to do so on the test distribution $Q$.

This failure is surprising because the test samples only contain sprites already encountered during training. The novelty lies solely in the combination of these sprites. We would expect any model that comprehends the compositional nature of the dataset to readily generalize to these test samples.

This compositional aspect of the generalization problem manifests itself in the structure of the training and test distribution. In our running example, the model was trained on samples from a distribution $P$ that contained 
all possible sprites in each slot, but only in combination with one base sprite in the other slot (illustrated in Figure~\ref{fig:comp_support}A). More formally, the support of $P$ can be written as
\begin{equation}\label{eq:orthogonal}
    \supp P = \left\{ (\boldsymbol z_1 \in \mathcal Z_1, \boldsymbol z_2 \in \mathcal Z_2) | \boldsymbol z_1 = \boldsymbol z_1^0 \vee \boldsymbol z_2 = \boldsymbol z_2^0 \right\}.
\end{equation}
The test distribution $Q$ is a uniform distribution over the full product space $\mathcal Z_1\times \mathcal Z_2$, i.e. it contains all possible sprite combinations. More generally, we say that a generalization problem is compositional if the test distribution contains only components that have been present in the training distribution, see Figure~\ref{fig:comp_support}. This notion can be formalized as follows based on the support of the marginal distributions:

\begin{definition}[Compositional support]\label{def:comp_support}
Given two arbitrary distribution $P, Q$ over latents $\boldsymbol z = (\boldsymbol z_1, ..., \boldsymbol z_K) \in \mathcal Z = \mathcal Z_1 \times \cdots \times \mathcal Z_K$, $P$ has \textit{compositional support} \wrt $Q$ if
\begin{equation}
    \supp P_{\boldsymbol z_k} = \supp Q_{\boldsymbol z_k} \subseteq \mathcal Z_k \quad\forall k \in [K].
\end{equation}
\end{definition}


Clearly, \textit{compositional generalization} requires compositional support. If regions of the test latent space exist for which a component is not observed, as in Figure~\ref{fig:comp_support}E, we can examine a model's generalization capability, but the problem is not compositional. Depending on whether the gap in the support is in the middle of a latent's domain or towards either end, the generalization problem becomes an \textit{interpolation} or \textit{extrapolation} problem instead, which are not the focus of this work.

\subsection{Sufficient conditions for compositional generalization}\label{sec:conditions}
With the above setup, we can now begin to examine under what conditions compositional generalization can be guaranteed to occur.

To make this question precise, let us assume for the moment that the sprites don't occlude each other but that they are just summed up in pixel space. Then the compositional representation of the generative process is simply $\{\boldsymbol{\mathrm {Id}}, \boldsymbol{\varphi}_1, \boldsymbol{\varphi}_2\}$, i.e.
\begin{equation}
    \label{eq:sumcase}
    \boldsymbol{f}(\boldsymbol{z}) = \boldsymbol{\varphi}_1(\boldsymbol z_1) + \boldsymbol{\varphi}_2(\boldsymbol z_2).
\end{equation}

The question becomes: Given supervised samples $(\boldsymbol{z}_i, \boldsymbol{x}_i)$ from $P$, can we learn a new model $\boldsymbol{\hat f}$ that is equivalent to $\boldsymbol{f}$ on $Q$, i.e. for which $\boldsymbol{\hat f}\equiv_{Q} \boldsymbol{f}$? We assume that $\boldsymbol{C}$ is known, so in order to generalize, we must be able to reconstruct the individual component functions $\boldsymbol{\varphi}_i$.

For the simple case from equation~\ref{eq:sumcase}, we can fully reconstruct the component functions as follows. First, we note that if $\supp P$ is in an open set, we can locally reconstruct the hidden Jacobian of $\boldsymbol{\varphi}_i$ from the observable Jacobian of $\boldsymbol{f}$ as
\begin{equation}
    \frac{\partial\boldsymbol{f}}{\partial\boldsymbol{z}_k}(\boldsymbol{z}) = \frac{\partial\boldsymbol{\varphi}_k}{\partial\boldsymbol{z}_k}(\boldsymbol{z}_k).
\end{equation}
Since the training distribution contains all possible component configurations $\boldsymbol{z_i}$, we can reconstruct the Jacobian of $\boldsymbol{\varphi_i}$ in every point $\boldsymbol{z_i}$. Then we know everything about $\boldsymbol{\varphi_i}$ up to a global offset (which can be removed if there exists a known initial point for integration).


Our goal is to extend this approach to a maximally large set of composition functions $\boldsymbol{C}$. Our reasoning is straightforward if $\boldsymbol{C}$ is the identity, but what if we have occlusions or other nonlinear interactions between slots? What are general conditions on $\boldsymbol{C}$ and the support of the training distribution $P$ such that we can still reconstruct the individual component functions and thus generalize compositionally?

Let us now consider the sprites example with occlusions, and let us assume that the support of $P$ is basically a thin region around the diagonal; see Figure~\ref{fig:suff_support}~(left). In this case, the two sprites are always relatively similar, leading to large overlaps. It is impossible to reconstruct the full Jacobian of the occluded sprite from a single sample. Instead, we need a set of samples for which the background sprite is the same while the foreground sprite is in different positions; see Figure~\ref{fig:suff_support}~(right). With sufficient samples of this kind, we can observe all pixels of the background sprite at least once. Then reconstruction of the Jacobian of $\boldsymbol{\varphi_1}$ is possible again.

This line of thought brings us to a more general condition on the data-generating process: The composition function $\boldsymbol{C}$ and the support $P$ must be chosen such that the full Jacobian can be reconstructed for each component function for all component latents. We formally define the concept of \emph{sufficient support} below. Note that whether the support of $P$ is sufficient or not strongly depends on the choice of composition function $\boldsymbol{C}$.

\begin{definition}[Sufficient support]\label{def:suff_support}
A distribution $P$ over latents $\boldsymbol z = (\boldsymbol z_1, ..., \boldsymbol z_K) \in \mathcal Z$, has \textit{sufficient support} \wrt a compositional representation of a function $\boldsymbol f$, if $\supp P$ is in an open set and for any latent value $\boldsymbol z_k^*$, there exists a (finite) set of points $P'_k(\boldsymbol z_k^*) \subseteq \{\boldsymbol p \in \supp P| \boldsymbol p_k = \boldsymbol z_k^* \}$ for which the sum of total derivatives of $\boldsymbol C$ has full rank. That is,
\begin{equation}
    \rank \sum_{\boldsymbol p \in P'_k(\boldsymbol z_k^*)} \frac{\partial \boldsymbol C}{\partial \boldsymbol \varphi_k}\big(\boldsymbol \varphi(\boldsymbol p)\big) = M,
\end{equation}
where $M$ is the dimension of the component space $\mathcal{\tilde X}_k \subseteq \mathbb R^M$.
\end{definition}

We are now ready to state our main theorem, namely that if $\boldsymbol{f}, \boldsymbol{\hat f}$ share the same composition function and if $P$ has compositional and sufficient support, then the model $\boldsymbol{\hat f}$ generalizes to $Q$ if it matches the ground-truth data-generating process $\boldsymbol{f}$ on $P$.

\begin{figure}[t]
  \centering
  \includegraphics[width=0.9\textwidth]{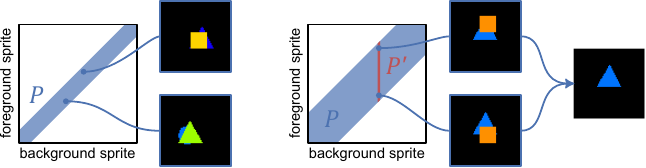}
  \caption{Illustration of the \textit{sufficient support} condition. For a (compositional) diagonal support, all samples will contain sprites with similar positions, leading to heavy occlusions and making reconstruction of the background sprite impossible (left). Reconstruction of the background sprite is only possible if the support is chosen broad enough, such that the subset of points sharing the same background sprite $P'$ contains samples with sufficient variance in the foreground sample. Specifically, each pixel of the background sprite must be observable at least once (right).}\label{fig:suff_support}
\end{figure}

\begin{theorem}\label{theorem}
Let  $P$, $Q$ be arbitrary distributions over latents $\boldsymbol z = (\boldsymbol z_1, ..., \boldsymbol z_K) \in \mathcal Z$.
Let $\boldsymbol f, \boldsymbol{\hat f}$ be functions with \emph{compositional representations} in the sense of definition~\ref{def:function_class} that share $\{ \boldsymbol C, \mathcal Z_1, ..., \mathcal Z_K \}$, but use arbitrary $\{ \boldsymbol{\varphi}_1, ..., \boldsymbol{\varphi}_K, \mathcal{\tilde X}_1, ..., \mathcal{\tilde X}_K \}, \{ \boldsymbol{\hat\varphi}_1, ..., \boldsymbol{\hat\varphi}_K, \mathcal{\hat X}_1, ..., \mathcal{\hat X}_K \}$.

Assume the following assumptions hold:
\begin{enumerate}[label=(A\arabic*)]
    \item\label{ass:well_behaved} $\boldsymbol C, \boldsymbol \varphi_k, \boldsymbol{\hat \varphi}_k$ are differentiable, $\boldsymbol C$ is Lipschitz in $\boldsymbol\varphi$, and $\boldsymbol\varphi$ is continuous in $\boldsymbol z$.
    \item\label{ass:comp_support} $P$ has \emph{compositional support} \wrt $Q$ in the sense of definition~\ref{def:comp_support}.
    \item\label{ass:suff_support} $P$ has \emph{sufficient support} \wrt $\boldsymbol f$ in the sense of definition~\ref{def:suff_support}.
    \item\label{ass:initial_point} There exists an initial point $\boldsymbol p^0\in \supp P$ such that $\boldsymbol \varphi (\boldsymbol p^0) = \boldsymbol{\hat \varphi}(\boldsymbol p^0)$.
\end{enumerate}

Then $\boldsymbol{\hat f}$ generalizes to $Q$, \ie $\boldsymbol f \underset{P}{\equiv} \boldsymbol{\hat f} \implies \boldsymbol f \underset{Q}{\equiv} \boldsymbol{\hat f}$.
\end{theorem}

The proof follows roughly the intuition we developed above in that we show that the Jacobians of the component functions can be reconstructed everywhere. Bear in mind that this is simply a construction for the proof: The theorem holds whenever $\boldsymbol{\hat f}$ fits the output of $\boldsymbol{f}$ on the training distribution $P$, which we can achieve with standard supervised training and without access to the ground-truth Jacobians. It should also be emphasized that since the compositional representation is not unique, the theorem holds if there exists at least one for which the assumptions are fulfilled. Note also that the initial point condition \ref{ass:initial_point} is needed in the proof, but in all practical experiments (see below), we can generalize compositionally without explicit knowledge of that point. We relegate further details to Appendix~\ref{app:proof}.


\section{Experiments}\label{sec:experiments}
We validate our theoretical framework on the multi-sprite data. All models were trained for 2000 epochs on training sets of 100k samples using an NVIDIA\,RTX\,2080\,Ti; all test sets contain 10k samples. Table~\ref{tab:experiments} summarizes the reconstruction quality achieved on the in-domain (ID) test set ($P$) and the entire latent space ($Q$) for all experiments.

\paragraph{Motivating experiment}\label{sec:main_experiment}
We implement the setup from Figure~\ref{fig:motivation} to demonstrate that a compositional model does indeed generalize if the conditions from Theorem~\ref{theorem} are met. We model the component functions as four fully-connected layers followed by four upsampling-convolution stages, mapping the 5d component latent to $64\times64$ RGB images. For training stability, the composition function is implemented as a soft pixel-wise addition using the sigmoid function $\sigma(\cdot)$ as
\begin{equation}\label{eq:add_sigmoid}
    \boldsymbol x = \sigma(\boldsymbol{\tilde x}_1) \cdot \boldsymbol{\tilde x}_1 + \sigma(-\boldsymbol{\tilde x}_1) \cdot \boldsymbol{\tilde x}_2,
\end{equation}
which allows component~1 to occlude component~2. We contrast this to a non-compositional \textit{monolithic} model, which has the same architecture as a single component function (with adjusted layer sizes to match the overall parameter count of the compositional model). We show that both models have the capacity to fit the data by training on random samples covering the entire latent space (Table~\ref{tab:experiments}, \textbf{\#1,2}). We then train on a distribution with orthogonal support as in equation~\ref{eq:orthogonal}, albeit with two planes for the foreground component to satisfy the sufficient support condition (Definition~\ref{def:suff_support}) as explained in Figure~\ref{fig:suff_support}. Both models can reconstruct ID samples, but only the compositional model generalizes to the entire latent space (Table~\ref{tab:experiments}, \textbf{\#3,4}).


\begin{table}[t]
    \centering
    \sisetup{
        tight-spacing=true,
        detect-family=true,
        detect-weight=true,
        mode=text,
        output-exponent-marker=\text{e}
    }
    \setlength{\tabcolsep}{0.4em}
    \renewcommand{\bfseries}{\fontseries{b}\selectfont} 
    \newrobustcmd{\B}{\bfseries}
    \begin{tabular}{r l l S[table-format=1.2e-1] @{\hspace{1pt}} l S[table-format=1.2e-1] @{\hspace{1pt}} l S[table-format=1.3] @{\hspace{1pt}} l S[table-format=-1.3] @{\hspace{1pt}} l }
        \toprule
        \B \# & \B Train Set & \B Model & \multicolumn{2}{c}{\B MSE ID $\scriptstyle\downarrow$} & \multicolumn{2}{c}{\B MSE all $\scriptstyle\downarrow$} & \multicolumn{2}{c}{\B $R^2$ ID $\scriptstyle\uparrow$} & \multicolumn{2}{c}{\B $R^2$ all $\scriptstyle\uparrow$} \\
        \midrule
         1 & Random     & Monolithic    & 1.73e-3 & \sd{1.47e-05} & 1.73e-3 & \sd{1.47e-05} & 0.931 & \sd{5.75e-04} & 0.931 & \sd{5.83e-04} \\
         2 & Random     & Compositional & 1.07e-03 & \sd{2.57e-05} & 1.07e-03 & \sd{2.57e-05} & 0.957 & \sd{1.02e-03} & 0.957 & \sd{1.02e-03} \\
         3 & Orthogonal & Monolithic    & 8.49e-04 & \sd{2.89e-05} & 4.06e-02 & \sd{3.86e-03} & 0.948 & \sd{1.71e-03} & -0.500 & \sd{6.70e-02} \\
         4 & Orthogonal & Compositional & 6.94e-04 &	\sd{1.09e-05} &	1.24e-03 &	\sd{4.11e-05} &	0.957 &	\sd{6.41e-04} &	0.951 &	\sd{1.43e-03}\\
        \midrule
         5 & Ortho.$\sim\mathcal N$ & Compositional & 7.01e-04 &	\sd{8.69e-06} &	1.24e-03 &	\sd{2.56e-05} &	0.957 &	\sd{5.51e-04} &	0.951 &	\sd{1.00e-03} \\
         6 & Diagonal & Compositional & 8.87e-04 &	\sd{1.03e-04} &	1.39e-03 &	\sd{4.04e-04} &	0.954 &	\sd{5.39e-03} &	0.945 &	\sd{1.61e-02} \\
         7 & Ortho. (broad) & Compositional & 6.50e-04 &	\sd{1.52e-05} &	1.16e-03 &	\sd{3.32e-05} &	0.959 &	\sd{9.67e-04} &	0.954 &	\sd{1.31e-03} \\
         8 & Diag. (broad) & Compositional & 9.51e-04 &	\sd{2.78e-05} &	1.13e-03 &	\sd{3.10e-05} &	0.957 &	\sd{1.24e-03} &	0.955 &	\sd{1.23e-03} \\
        \midrule
         9 & Ortho. (gap) & Compositional & 7.36e-04 &	\sd{2.70e-05} &	2.64e-03 &	\sd{1.81e-04} &	0.954 &	\sd{1.68e-03} &	0.895 &	\sd{7.40e-03} \\
        10 & Diag. (narrow) & Compositional & 2.22e-03 &	\sd{1.24e-03} &	1.04e-02 &	\sd{7.10e-03} &	0.867 &	\sd{7.45e-02} &	0.589 &	\sd{2.82e-01} \\
        11 & Orthogonal & Comp. (RGBa) & 2.67e-04 &	\sd{2.81e-06} &	5.24e-04 &	\sd{3.68e-06} &	0.984 &	\sd{1.80e-04} &	0.979 &	\sd{1.44e-04}\\
        \bottomrule
    \end{tabular}
    \vspace{10pt}
    \caption{We report the reconstruction quality measured as mean squared error (MSE, lower is better) and variance-weighted $R^2$ score (closer to 1 is better) for both the in-domain (ID) test set and the entire latent space, averaged over 5 random seeds. \textbf{\#1-4} The results demonstrate that a \textit{monolithic} model fails to generalize in the setup from Figure~\ref{fig:motivation}, but a \textit{compositional} model performs well on the entire latent space. \textbf{\#5-8} Generalization can occur in a variety of settings that fulfill the sufficient conditions from Theorem~\ref{theorem}. \textbf{\#9,10} Violating the compositional and sufficient support condition prohibits generalization, while choosing a more complex function class still works (\textbf{\#11}).}\label{tab:experiments}
\end{table}

\paragraph{Flexible compositional support} Next, we demonstrate the variety of settings that fulfil the compositional support assumption as illustrated in Figure~\ref{fig:comp_support}B and C. To this end, we repeat the experiment on training sets $P$ sampled from (i) a normal distribution with orthogonal support (Table~\ref{tab:experiments}, \textbf{\#5}) and (ii) a uniform distribution over a diagonal support chosen broad enough to satisfy the sufficient support condition (Table~\ref{tab:experiments}, \textbf{\#6}). The model generalizes to the entire latent space in both settings. Since the generalization performance is already close to ceiling, broadening the support of both distributions (Table~\ref{tab:experiments}, \textbf{\#7,8}) does not further increase performance.

\paragraph{Violating Conditions} Finally, we look at the effect of violating some conditions.

\begin{itemize}[leftmargin=12pt]
    \item \textbf{Gaps in support} (Table~\ref{tab:experiments}, \textbf{\#9}) If there are gaps in the support of the training set such that some component configurations are never observed (Figure~\ref{fig:comp_support}E) violates the compositional support condition (Definition~\ref{def:comp_support}). While the overall reconstruction performance only drops slightly, visualizing the reconstruction error over a 2d-slice of the latent space in Figure~\ref{fig:heatmap} illustrates clearly that generalization fails exactly where the condition is violated.
    
    \item \textbf{Insufficient training variability} (Table~\ref{tab:experiments}, \textbf{\#10}) Reducing the width of the diagonal support violates the sufficient support condition (Definition~\ref{def:suff_support}) as soon as some parts of the background component are always occluded and can not be observed in the output anymore. We can clearly see that reconstruction performance on the entire latent space drops significantly as a result.
    
    \item \textbf{Collapsed Composition Function} (Table~\ref{tab:experiments}, \textbf{\#11}) Changing the output of each component function from RGB to RGBa and implementing the composition as alpha compositing yields a model that is still compositional, but for which no support can satisfy the sufficient support condition since the derivative of transparent pixels will always be zero and the Jacobian matrix can therefore never have full rank (more details in Appendix~\ref{app:theory_details}). However, we observe that the model still generalizes to the entire latent space and achieves even lower reconstruction error than the original model. This emphasizes that what we present are merely \textit{sufficient} conditions, which might be loosened in future work.
\end{itemize}

\begin{figure}
    \centering
    \includegraphics[width=0.6\textwidth]{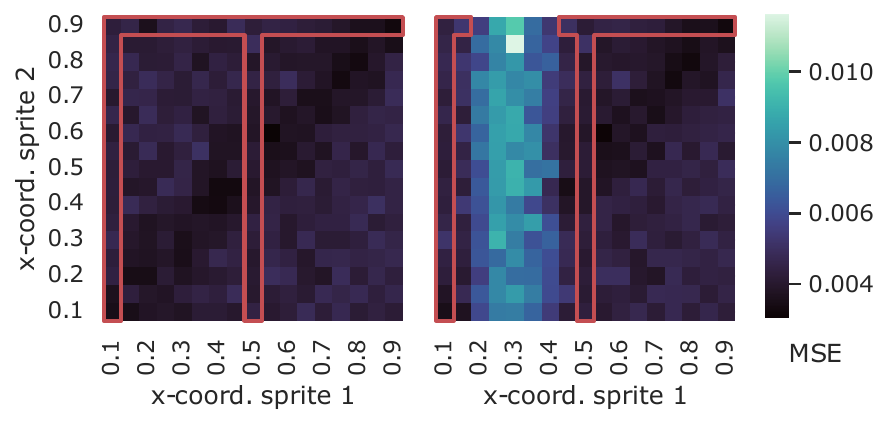}
    \caption{Heatmap of the reconstruction error over a $z_{1,\text{x}}$-$z_{2,\text{x}}$-projection of the latent space with overlaid training support (red). Generalization can occur when the support is compositional (left) but fails exactly where the support is incomplete at $z_{1,\text{x}} \in [0.14, 0.46]$ (right).}\label{fig:heatmap}
\end{figure}

\section{Discussion}\label{sec:discussion}

We presented a first step and a framework to study compositional generalization in a more principled way. Clearly, there remain many open questions and limitations that we leave for future work.

\paragraph{Supervised setting}
We only studied a supervised regression setting in which the model has access to the ground-truth latents of each training sample. Ultimately, we are
interested in the unsupervised setting akin to what is typically studied in identifiable representation learning. The unsupervised setting comes with inherent ambiguities that make generalizations guarantees harder to derive. Still, the results in this paper build an important foundation for future studies because sufficient conditions in the supervised setting can be considered necessary conditions in the unsupervised setting.

\paragraph{Jacobian and initial point} The proof of Theorem~\ref{theorem} utilizes the Jacobian of the ground-truth model. We emphasize again that this construction is necessary only for the proof and does not mean that we require access to the data-generating processes' full Jacobian for training.
Similarly, the existence of an initial point $\boldsymbol{p^0}$ is a technicality of the proof that is not reflected in the experiments. While it is not yet clear whether it is possible to complete the proof without the initial point condition,  we believe there is a self-consistency condition that might alleviate the need for this condition. The experiments thus hint at the existence of alternative proof strategies with relaxed assumptions.

\paragraph{Known composition function}
We also assume the composition function to be known which is approximately true in many interesting scenarios, such as object composition in scenes or the composition of instruments in music. In fact, many structured representation learning approaches like \eg SlotAttention~\cite{locatelloObjectCentricLearningSlot2020} incorporate structural components that are meant to mimic the compositional nature of the ground-truth-generating process. In other interesting cases like language, however, the composition function is unknown a priori and needs to be learned. This might be possible by observing how the gradients of $\boldsymbol{C}$ change with respect to a fixed slot, at least if certain regularity conditions are fulfilled. 




\paragraph{Inductive biases}
Some of the conditions we derived can be relaxed in the presence of certain inductive biases. For example, models with an inductive bias towards shift invariance might be able to cope with certain gaps in the training support (e.g., if sprites are not visible in every position). Similarly, assuming all component functions $\boldsymbol{\varphi}$ to be identical would substantially simplify the problem and allow for much smaller sufficient supports $P$. The conditions we derived do not assume any inductive bias but are meant to formally guarantee compositional generalization. We expect that our conditions generalize to more realistic conditions as long as the core aspects are fulfilled.

\paragraph{Error bounds} Our generalization results hold only if the learned model perfectly matches the ground-truth model on the training distribution. This is similar to identifiable representation learning, where a model must find the global minimum of a certain loss or reconstruction error for the theory to hold. Nonetheless, extending our results towards generalization errors that are bounded by the error on the training distribution is an important avenue for future work.

\paragraph{Broader impact} Compositional generalization, once achieved, has the potential to be beneficial in many downstream applications. By substantially increasing sample and training efficiency, it could help to democratize the development and research of large-scale models. Better generalization capabilities could also increase the reliability and robustness of models but may amplify existing biases and inequalities in the data by generalizing them and hinder our ability to interpret and certify a model's decisions. 

\section{Conclusion}

Machine learning, despite all recent breakthroughs, still struggles with generalization. Taking advantage of the basic building blocks that compose our visual world and our languages remains unique to human cognition. 
We believe that progress towards more generalizable machine learning is hampered by a lack of a formal understanding of how generalization can occur. This paper focuses on compositional generalization and provides a precise mathematical framework to study it. We derive a set of sufficient conditions under which compositional generalization can occur and which cover a wide range of existing approaches. We see this work as a stepping stone towards identifiable representation learning techniques that can provably infer and leverage the compositional structure of the data. It is certainly still a long road toward scalable empirical learning techniques that can fully leverage the compositional nature of our world. However, once achieved, there is an opportunity for drastically more sample-efficient, robust, and human-aligned machine learning models.
\section*{Acknowledgments}
We would like to thank (in alphabetical order): Jack Brady, Simon Buchholz, Attila Juhos, and Roland Zimmermann for helpful discussions and feedback.

This work was supported by the German Federal Ministry of Education and Research (BMBF): Tübingen AI Center, FKZ: 01IS18039A. WB acknowledges financial support via an Emmy Noether Grant funded by the German Research Foundation (DFG) under grant no. BR 6382/1-1 and via the Open Philantropy Foundation funded by the Good Ventures Foundation. WB is a member of the Machine Learning Cluster of Excellence, EXC number 2064/1 – Project number 390727645. This research utilized compute resources at the Tübingen Machine Learning Cloud, DFG FKZ INST 37/1057-1 FUGG. We thank the International Max Planck Research School for Intelligent Systems (IMPRS-IS) for supporting TW and PM.

\newpage
\section*{Author contributions}
The project was led and coordinated by TW. TW and PM jointly developed the theory with insights from WB. TW implemented and conducted the experiments with input from PM and WB. TW led the writing of the manuscript with help from WB, PM, and MB. TW created all figures with comments from PM and WB.
\bibliographystyle{unsrtnat}
\bibliography{references,manual_references}

\newpage
\setcounter{section}{0}
\renewcommand\thesection{\Alph{section}}
\section{Proof of Theorem \ref{theorem}}\label{app:proof}
We reiterate the setup and notation introduced in the paper here for ease of reference.

\paragraph{Notation}
$[N]$ denotes the set of natural number $\{1, 2, ..., N\}$.
$\boldsymbol{\mathrm {Id}}$ denotes the (vector-valued) identity function.
We write two functions $f, g$ agreeing for all points in set $P$ as $f \equiv_P g$.
Finally, we write the total derivative of a vector-valued function $\boldsymbol f$ by all its inputs $\boldsymbol z$ as $\frac{\partial \boldsymbol f}{\partial \boldsymbol z}$, \ie the Jacobian matrix with entries $\frac{\partial f_i}{\partial z_j}$.

\paragraph{Setup}
We are given two arbitrary distributions $P, Q$ over latents $\boldsymbol z = (\boldsymbol z_1, ..., \boldsymbol z_K) \in \mathcal Z$. Each latent $\boldsymbol z_k$ describes one of the $K$ \textit{components} of the final data point $\boldsymbol x$ produced by the ground-truth data-generating process $\boldsymbol f$. A model $\boldsymbol{\hat f}$ is trained to fit the data-generating process on samples of $P$; the aim is to derive conditions on $P$ and $\boldsymbol{\hat f}$ that are sufficient for $\boldsymbol{\hat f}$ to then also fit $\boldsymbol f$ on $Q$.

We assume that $\boldsymbol f, \boldsymbol{\hat f}$ are chosen such that we can find at least one \textit{compositional representation} (Definition~\ref{def:function_class}) for either function that shares a common \textit{composition function} $\boldsymbol C$ and factorization of the latent space $\mathcal Z_1 \times \cdots \times \mathcal Z_K = \mathcal Z$.

\begin{proof}[Proof of Theorem~\ref{theorem}]
For $\boldsymbol{\hat f}$ to generalize to $Q$, we need to show fitting $\boldsymbol f$ on $P$ implies also fitting it on $Q$, in other words
\begin{equation}\label{eq:goal}
    \boldsymbol f \underset{P}{\equiv} \boldsymbol{\hat f} \implies \boldsymbol f \underset{Q}{\equiv} \boldsymbol{\hat f}
\end{equation}

\proofpart\label{pr:phi_q_to_f_q}
Since $\boldsymbol C$ is the same for both functions, we immediately get
\begin{equation}\label{eq:phi_q_to_f_q}
    \boldsymbol \varphi \underset{Q}{\equiv} \boldsymbol{\hat \varphi} \implies \boldsymbol f \underset{Q}{\equiv} \boldsymbol{\hat f},
\end{equation}
\ie it suffices to show that the \textit{component functions} generalize.
Note, however, that since $\boldsymbol C$ is not generally assumed to be invertible, we do \textit{not} directly get that agreement of $\boldsymbol f, \boldsymbol{\hat f}$ on $P$ also implies agreement of their component functions $\boldsymbol \varphi, \boldsymbol{\hat \varphi}$ on $P$.

\proofpart\label{pr:phi_p_to_phi_q}
We require $P$ to have \textit{compositional support} \wrt $Q$ (Definition~\ref{def:comp_support} and Assumption~\ref{ass:comp_support}).
The consequence of this assumption is that any point $\boldsymbol q = (\boldsymbol q_1, ..., \boldsymbol q_K) \in Q$ can be constructed from components of the $K$ \textit{support points} $\boldsymbol p^k = \left(\boldsymbol p^k_1, ..., \boldsymbol p^k_K\right) \in P$ subject to $\boldsymbol p^k_k = \boldsymbol q_k$ as
\begin{equation}
    \boldsymbol q = \left( \boldsymbol p^1_1, ..., \boldsymbol p^K_K \right).
\end{equation}
A trivial consequence, then, is that points $\boldsymbol{\tilde x} \in \mathcal{\tilde X}$ in \textit{component space} corresponding to points in $Q$ in latent space can always be mapped back to latents in $P$
\begin{equation}
    \boldsymbol \varphi(\boldsymbol q) = \big(\boldsymbol \varphi_1(\boldsymbol q_1), ..., \boldsymbol \varphi_K(\boldsymbol q_K)\big) = \left(\boldsymbol \varphi_1\left(\boldsymbol p^{(1)}_1\right), ..., \boldsymbol \varphi_K\left( \boldsymbol p^{(K)}_K \right)\right)
\end{equation}
because each \textit{component function} $\boldsymbol \varphi_k$ only depends on the latents $\boldsymbol z_k$ of a single component.
This is also the case for the component functions $\boldsymbol{\hat \varphi}$ of $\boldsymbol{\hat f}$ so that we get
\begin{equation}\label{eq:phi_p_to_phi_q}
    \boldsymbol \varphi \underset{P}{\equiv} \boldsymbol{\hat \varphi} \implies \boldsymbol \varphi \underset{Q}{\equiv} \boldsymbol{\hat \varphi}.
\end{equation}

\proofpart\label{pr:f_p_to_dphi_p}
We now only need to show that $\boldsymbol \varphi \underset{P}{\equiv} \boldsymbol{\hat \varphi}$ follows from $\boldsymbol f \underset{P}{\equiv} \boldsymbol{\hat f}$. As noted above, this is not guaranteed to be the case, as $\boldsymbol C$ is not generally invertible (\eg in the presence of occlusions). We, therefore, need to consider when a unique reconstruction of the component functions $\boldsymbol \varphi$ (and correspondingly $\boldsymbol{\hat \varphi}$) is possible, based on only the observations $\boldsymbol x = \boldsymbol f(\boldsymbol z)$ on $Q$.

As explained in the main paper, we can reason about how a change in the latents $\boldsymbol z_k$ of some slot affects the final output, which we can express through the chain rule as
\begin{equation}\label{eq:chain_rule}
    {\color{gray}\underbrace{\color{black} \frac{\partial \boldsymbol f}{\partial \boldsymbol z_k}(\boldsymbol z) }_{N \times D}}
    = {\color{gray}\underbrace{\color{black} \frac{\partial \boldsymbol C}{\partial \boldsymbol \varphi_k}\big(\boldsymbol \varphi(\boldsymbol z)\big) }_{N \times M}}
    {\color{gray}\underbrace{\color{black} \frac{\partial \boldsymbol \varphi_k}{\partial \boldsymbol z_k}(\boldsymbol z_k) }_{M \times D}} \quad\forall k \in [K].
\end{equation}
Here, $N$ is the dimension of the final output (\eg $64 \times 64 \times 3$ for RGB images), $M$ is the dimension of a component's representation $\boldsymbol{\tilde x}_k$ (\eg also $64 \times 64 \times 3$ for RGB images), and $D$ is the dimension of a component's latent description $\boldsymbol z_k$ (\eg 5: x-position, y-position, shape, size, hue for sprites).
Note that we can look at the derivative component-wise because each \textit{component function} $\boldsymbol \varphi_k$ only depends on the latents $\boldsymbol z_k$ of its component. However, the \textit{combination function} still depends on the (hidden) representation of all components, and therefore $\frac{\partial \boldsymbol C}{\partial \boldsymbol \varphi_k}$ is a function of all $\boldsymbol \varphi$ and the entire $\boldsymbol z$.

In equation~\ref{eq:chain_rule}, the left-hand side (LHS) $\frac{\partial \boldsymbol f}{\partial \boldsymbol z_k}$ can be computed from the training, as long as $\supp P$ is an open set. On the right-hand side (RHS), the functional form of $\frac{\partial \boldsymbol C}{\partial \boldsymbol \varphi_k}$ is known since $\boldsymbol C$ is given, but since $\boldsymbol \varphi(\boldsymbol z)$ is still unknown, the exact entries of this Jacobian matrix are unknown. As such, equation~\ref{eq:chain_rule} defines a system of partial differential equations (PDEs) for the set of component functions $\boldsymbol \varphi$ with independent variables $\boldsymbol z$.

Before we can attempt to solve this system of PDEs, we simplify it by isolating $\frac{\partial \boldsymbol \varphi_k}{\partial \boldsymbol z_k}$. Since all terms are matrices, this is equivalent to solving a system of linear equations. For $N = M$, $\frac{\partial \boldsymbol C}{\partial \boldsymbol \varphi_k}$ is square, and we can solve by taking its inverse as long as the determinant is not zero. In the general case of $N \geq M$, however, we have to resort to the pseudoinverse to write
\begin{equation}\label{eq:LSE_solution}
    \frac{\partial \boldsymbol \varphi_k}{\partial \boldsymbol z_k}^*
    = \left( \frac{\partial \boldsymbol C}{\partial \boldsymbol \varphi_k}^\top \frac{\partial \boldsymbol C}{\partial \boldsymbol \varphi_k} \right)^{-1} \frac{\partial \boldsymbol C}{\partial \boldsymbol \varphi_k}^\top
    \frac{\partial \boldsymbol f}{\partial \boldsymbol z_k}
    \quad\forall k \in [K],
\end{equation}
which gives all solutions $\frac{\partial \boldsymbol \varphi_k}{\partial \boldsymbol z_k}^*$ if any exist. This system is overdetermined, and a (unique) solution exists if $\frac{\partial \boldsymbol C}{\partial \boldsymbol \varphi_k}$ has full (column) rank. In other words, to execute this simplification step on $P$, we require that for all $\boldsymbol z \in P$ the $M$ column vectors of the form
\begin{equation}\label{eq:column}
    \left( \frac{\partial C_1}{\partial \varphi_{km}}\big(\boldsymbol \varphi(\boldsymbol z)\big), ..., \frac{\partial C_N}{\partial \varphi_{km}}\big(\boldsymbol \varphi(\boldsymbol z)\big) \right)^\top
    \quad\forall m \in [M]
\end{equation}
are linearly independent. Each entry of a column vector describes how all entries $C_n$ of the final output (\eg the pixels of the output image) change with a single entry $\varphi_{km}$ of the intermediate representation of component $k$ (\eg a single pixel of the component-wise image). It is easy to see that if even a part of the intermediate representation is not reflected in the final output (\eg in the presence of occlusions, when a single pixel of one component is occluded), the entire corresponding column is zero, and the matrix does not have full rank.

To circumvent this issue, we realize that the LHS of equation~\ref{eq:LSE_solution} only depends on the latents $\boldsymbol z_k$ of a single component. Hence, for a given latent $\boldsymbol z$ and a slot index $k$, the correct component function will have the same solution for all points in any (finite) set
\begin{equation}
    P'(\boldsymbol z, k) \subseteq \left\{ \boldsymbol p \in \supp P | \boldsymbol p_k = \boldsymbol z_k \right\}.
\end{equation}
We can interpret these points as the intersection of $P$ with a plane in latent space at $\boldsymbol z_k$ (\eg all latent combinations in the training set in which one component is fixed in a specific configuration).
We can then define a modified composition function $\boldsymbol{\tilde C}$ that takes $\boldsymbol z$ and a slot index $k$ as input and produces a ``superposition'' of images corresponding to the latents in the subset as
\begin{equation}\label{eq:superposition}
    \boldsymbol{\tilde C}\left( \boldsymbol z, k \right) = \sum_{\boldsymbol p \in P'(\boldsymbol z, k)} \boldsymbol C\big(\boldsymbol \varphi(\boldsymbol p) \big).
\end{equation}
Essentially, we are condensing the information from multiple points in the latent space into a single function.
This enables us to write a modified version of equation~\ref{eq:chain_rule} as
\begin{equation}
    \sum_{\boldsymbol p \in P'(\boldsymbol z, k)} \frac{\partial \boldsymbol f}{\partial \boldsymbol z_k}(\boldsymbol p)
    = \sum_{\boldsymbol p \in P'(\boldsymbol z, k)} \frac{\partial \boldsymbol C}{\partial \boldsymbol \varphi_k}\big(\boldsymbol \varphi(\boldsymbol p)\big)
    \frac{\partial \boldsymbol \varphi_k}{\partial \boldsymbol z_k}(\boldsymbol z_k)
    = \frac{\partial \boldsymbol{\tilde C}}{\partial \boldsymbol \varphi_k} (\boldsymbol z, k)
    \frac{\partial \boldsymbol \varphi_k}{\partial \boldsymbol z_k}(\boldsymbol z_k)
    \quad\forall k \in [K]
\end{equation}
Now we can solve for $\frac{\partial \boldsymbol \varphi_k}{\partial \boldsymbol z_k}$ as in equation~\ref{eq:LSE_solution}, but this time require only that $\frac{\partial \boldsymbol{\tilde C}}{\partial \boldsymbol \varphi_k}$ has full (column) rank for a unique solution to exist, \ie
\begin{equation}
    \rank \frac{\partial \boldsymbol{\tilde C}}{\partial \boldsymbol \varphi_k} (\boldsymbol z, k)
    = \sum_{\boldsymbol p \in P'(\boldsymbol z, k)} \frac{\partial \boldsymbol C}{\partial \boldsymbol \varphi_k}\big(\boldsymbol \varphi(\boldsymbol p)\big)
    = M \quad \forall \boldsymbol z \in P \quad \forall k \in [K].
\end{equation}
In general, this condition is easier to fulfill since full rank is not required in any one point but over a set of points. For occlusions, for example, any pixel of one slot can be occluded in some points $\boldsymbol p \in P'$, as long as it is not occluded in all of them. We can interpret this procedure as ``collecting sufficient information'' such that an inversion of the generally non-invertible $\boldsymbol C$ becomes feasible locally.

The requirement that $\supp P$ has to be an open set, together with the full rank condition on the Jacobian of the composition function condensed over multiple points, $\boldsymbol{\tilde C}$, is termed \textit{sufficient support} in the main paper (Definition~\ref{def:suff_support} and Assumption~\ref{ass:suff_support}). As explained here, this allows for the reconstruction of $\frac{\partial \boldsymbol \varphi_k}{\partial \boldsymbol z_k}$ from the observations, \ie
\begin{equation}\label{eq:f_p_to_dphi_p}
    \boldsymbol f \underset{P}{\equiv} \boldsymbol{\hat f} \implies \frac{\partial \boldsymbol \varphi}{\partial \boldsymbol z} \underset{P}{\equiv} \frac{\partial \boldsymbol{\hat \varphi}}{\partial \boldsymbol z}.
\end{equation}

\proofpart\label{pr:dphi_p_to_phi_p}
The above step only gives us agreement of the \textit{derivative} of the component functions, $\frac{\partial \boldsymbol \varphi_k}{\partial \boldsymbol z_k}$, not agreement of the functions themselves. As explained above, the solution to the linear system of equations~\ref{eq:LSE_solution} constitutes a system of partial differential equations (PDEs) in the set of component functions $\boldsymbol \varphi$ with independent variables $\boldsymbol z$. We can see that this system has the form
\begin{equation}
    \partial_i \boldsymbol \varphi(\boldsymbol z) = \boldsymbol a_i\big(\boldsymbol z, \boldsymbol \varphi(\boldsymbol z)\big),
\end{equation}
where $i \in [L] = [K\times D]$ is an index over the flattened dimensions $K$ and $D$ such that $\partial_i \boldsymbol \varphi$ denotes $\frac{\partial \boldsymbol\varphi}{\partial z_{L}}$ (which is essentially one column of $\frac{\partial \boldsymbol \varphi_k}{\partial \boldsymbol z_k}$ aggregated over all $k$) and $\boldsymbol a_i$ is the combination of corresponding terms from the LHS. If this system allows for more than one solution, we cannot uniquely reconstruct the component functions from their derivatives.

If we have access to some initial point, however, for which we know $\boldsymbol \varphi(\boldsymbol 0) = \boldsymbol \varphi^0$, we can write
\begin{equation}\label{eq:PDE}
\begin{aligned}
    \boldsymbol \varphi(z_1, ..., z_{L}) - \boldsymbol \varphi^* = &\big( \varphi(z_1, ..., z_{L}) - \varphi(0, z_2, ..., z_{L}) \big) \\
    &+ \big( \varphi(0, z_2, ..., z_{L}) - \varphi(0, 0, z_3, ..., z_{L}) \big) \\
    &+ ... \\
    &+ \big( \varphi(0, ..., 0, z_{L}) - \varphi(0, ..., 0) \big).
\end{aligned}
\end{equation}

In each line of this equation, only a single $z_i =: t$ is changing; all other $z_1, ..., z_{L}$ are fixed. Any solution of \ref{eq:PDE}, therefore, also has to solve the $L$ ordinary differential equations (ODEs) of the form
\begin{equation}
    \partial_t \boldsymbol \varphi(z_1, ..., z_{i-1}, t, z_{i+1}, ..., z_{L}) = \boldsymbol a_i\big(z_1, ..., z_{i-1}, t, z_{i+1}, ..., z_{L}, \boldsymbol \varphi(z_1, ..., z_{i-1}, t, z_{i+1}, ..., z_{L}) \big),
\end{equation}
which have a unique solution if $\boldsymbol a_i$ is Lipschitz in $\boldsymbol \varphi$ and continuous in $z_i$, as guaranteed by \ref{ass:well_behaved}. Therefore, \ref{eq:PDE} has at most one solution.
This reference point does not have to be in $\boldsymbol z = \boldsymbol 0$, as a simple coordinate transform will yield the same result for any point in $P$.
It is therefore sufficient that there exists \textit{some} point $\boldsymbol p^0 \in P$ for which $\boldsymbol \varphi (\boldsymbol p^0) = \boldsymbol{\hat \varphi}(\boldsymbol p^0)$ to obtain the same unique solution for $\boldsymbol \varphi$ and $\boldsymbol{\hat \varphi}$, which is exactly what \ref{ass:initial_point} states. Overall, this means that agreement of the derivatives of the component functions also implies agreement of the component functions themselves, \ie
\begin{equation}\label{eq:dphi_p_to_phi_p}
   \frac{\partial \boldsymbol \varphi}{\partial \boldsymbol z} \underset{P}{\equiv} \frac{\partial \boldsymbol{\hat \varphi}}{\partial \boldsymbol z} \implies \boldsymbol \varphi \underset{P}{\equiv} \boldsymbol{\hat \varphi}
\end{equation}

\proofpart
Finally, we can conclude the model $\boldsymbol{\hat f}$ fitting the ground-truth generating process $\boldsymbol f$ on the training distribution $P$, through \ref{eq:f_p_to_dphi_p}, \ref{eq:dphi_p_to_phi_p}, \ref{eq:phi_p_to_phi_q}, \ref{eq:phi_q_to_f_q}, implies the model generalizing to $Q$ as well. In other words, equation~\ref{eq:goal} holds.

\end{proof}

\section{Details about the compositional functions}\label{app:theory_details}
As explained in equation~\ref{eq:add_sigmoid} in section~\ref{sec:experiments}, the composition function is implemented as a soft pixel-wise addition in most experiments. The use of the sigmoid function $\sigma(\cdot)$ in the composition
\begin{equation}
    \boldsymbol x = \sigma(\boldsymbol{\tilde x}_1) \cdot \boldsymbol{\tilde x}_1 + \sigma(-\boldsymbol{\tilde x}_1) \cdot \boldsymbol{\tilde x}_2
\end{equation}
was necessary for training stability. With this formulation, sprites can also overlap somewhat transparently, which is not desired and leads to small reconstruction artifacts for some specific samples. Implementing the composition with a step function as
\begin{equation}
    \boldsymbol x = \operatorname{step}(\boldsymbol{\tilde x}_1) \cdot \boldsymbol{\tilde x}_1 + \operatorname{step}(-\boldsymbol{\tilde x}_1) \cdot \boldsymbol{\tilde x}_2
\end{equation}
instead would be more faithful to the ground-truth data-generating process, but is hard to train with gradient descent.

Note that both formulations could easily be extended to more than one sprite by simply repeating the composition operation with any additional sprite.

In section~\ref{sec:experiments}, we also looked at a model that implements the composition through alpha compositing instead (see also Table~\ref{tab:experiments}, \textbf{\#11}). Here, each component's intermediate representation is an RGBa image. The components are then overlaid on an opaque black background using the composition function
\begin{align}
    x_\alpha &= x_{1, \alpha} + \left( 1 - x_{1, \alpha} \right) \cdot x_{2, \alpha} \\
    x_\text{RGB} &= x_{1, \alpha} \cdot x_{1, \text{RGB}} + \left( 1 - x_{1, \alpha} \right) \cdot \frac{x_{2, \alpha}}{x_\alpha} \cdot x_{2, \text{RGB}}.
\end{align}
While this yields a compositional function, the sufficient support condition (Definition~\ref{def:suff_support}) is generally not fulfilled on the sprites data. The reason is that in fully transparent pixels ($\alpha = 0$), changing the RGB value is not reflected in the output. Conversely, if a pixel is black, changing its alpha value will not affect how it is blended over a black background. As a result, most columns in the Jacobian $\frac{\partial \boldsymbol C}{\partial \boldsymbol \varphi_k}$ (see also equation~\ref{eq:column}) will be zero. Since the intermediate representations of each sprite will contain a lot of black or transparent pixels (the entire background), the rank of the Jacobian here will be low. In this case, the workaround from equation~\ref{eq:superposition} does not help since the low rank is not a result of another component in the foreground but of the specific parameterization of each component itself.

As stated in the main paper, the fact that this parameterization still produces good results and generalizes well is an indicator that there might be another proof strategy or workaround that avoids this specific issue.

\end{document}